\documentclass{article}
\usepackage{ijcai19}

\usepackage{times}
\usepackage{soul}
\usepackage{url}
\usepackage[utf8]{inputenc}
\usepackage[small]{caption}
\usepackage{graphicx}
\usepackage{amsmath}
\usepackage{booktabs}
\usepackage{algorithm}
\usepackage{algorithmic}
\usepackage{tikz}
\usetikzlibrary{shapes.geometric, arrows}
\usepackage{etoolbox}\AtBeginEnvironment{algorithmic}{\small}

\usepackage{mathtools}
\usepackage{booktabs}

\DeclarePairedDelimiter\abs{\lvert}{\rvert}%
\makeatletter
\let\oldabs\abs
\def\abs{\@ifstar{\oldabs}{\oldabs*}}

\usepackage[utf8]{inputenc}
\usepackage{pgfplots}

\usepackage{subcaption}
\usepackage{gensymb}
\usepackage{amsmath,amsfonts,amssymb,amsthm,epsfig,epstopdf,titling,url,array}
\usepackage{enumerate}

\newcommand{\mathbbm}[1]{\text{\usefont{U}{bbm}{m}{n}#1}}
\newtheorem{thm}{Theorem}[section]
\newtheorem*{thmt*}{Theorem}
\newtheorem{lem}[thm]{Lemma}
\newtheorem{assumption}{Assumption}

\pgfplotsset{compat=1.14}

\pgfplotsset{
	discard if not/.style 2 args={
		x filter/.code={
			\edef\tempa{\thisrow{#1}}
			\edef\tempb{#2}
			\ifx\tempa\tempb
			\else
			
			\fi++
		}
	}
}

\newtheorem{defn}{Definition}[section]

\title{Multi-Differential Fairness Auditor for Black Box Classifiers}

\author{
  Gitiaux, Xavier\\
  \texttt{xgitiaux@gmu.edu}
  \and
  Rangwala, Huzefa\\
  \texttt{hrangwala@cs.gmu.edu}
}
\date{}

\begin{document}

\maketitle

  \begin{abstract}
  Machine learning algorithms are increasingly involved in sensitive decision-making process with adversarial implications on individuals. This paper presents  \textbf{mdfa}, an approach that identifies the characteristics of the victims of a classifier's discrimination. We measure discrimination as a violation of multi-differential fairness.  Multi-differential fairness is a guarantee that a black box classifier's outcomes do not leak information on the sensitive attributes of a small group of individuals. We reduce the problem of identifying worst-case violations to matching distributions and predicting where sensitive attributes and classifier's outcomes coincide. We apply \textbf{mdfa} to a recidivism risk assessment classifier  and demonstrate that individuals identified as African-American with little criminal history are three-times more likely to be considered at high risk of violent recidivism than similar individuals but not African-American.
\end{abstract}

\vspace{-1.5em}
\section{Introduction}

Machine learning algorithms are increasingly used to support decisions that could impose adverse consequences on an individual's life: for example in the  judicial system, algorithms are used to assess whether a criminal offender is likely to recommit crimes; or within banks  to determine the default risk of a potential borrower. At issue is whether classifiers are fair
\cite{calsamiglia2009decentralizing} i.e.,  whether classifiers' outcomes are independent of \textit{exogenous irrelevant characteristics} or sensitive
attributes like race and/or gender. Abundant examples of classifiers' discrimination can be found in diverse applications ( \cite{atlantic2016,ProPublica2016}). ProPublica (\cite{ProPublica2016}) reported that COMPAS machine learning based recidivism risk assessment tool assigns disproportionately higher risk to African-American defendants than to Caucasian defendants.

Establishing contestability is challenging for potential victims of machine learning discrimination because (i) many assessment tools are proprietary and usually not  transparent; and (ii) a precedent in United States case law  places the burden on the plaintiff to demonstrate disparate treatment -- to establish that characteristics irrelevant to the task affect the algorithm's outcomes (\textit{Loomis vs. the State of Wisconsin} \cite{Loomis}). Identifying the definitive characteristics of a classifier's discrimination empowers the victims of such discrimination.  Moreover,  a classifier's user needs warnings for individual instances in which severe profiling/discrimination has been detected. In this paper, we present a theoretical framework (multi-differential fairness) and a practical tool (\textbf{mdfa}) to  characterize the individuals who can make a strong claim for being discriminated against.

To demonstrate a classifier's disparate treatment, we need to separate its discrimination from the social biases already encoded in the data. We define a classifier as multi-differentially fair if there is no subset of the feature space for which the classifier's outcomes are dependent on sensitive attributes. For any sub-population, disparate treatment is then measured as the maximum-divergence distance between the prior and posterior distributions of sensitive attributes.

\textbf{mdfa} audits for the worst-case violations of multi-differential fairness. Theoretically, our construction relies on a reduction to both matching distribution and agnostic learning problems. First, in order to neutralize the social-biases encoded in the data, we re-balance it by minimizing the maximum-mean discrepancy between distributions conditioned on different sensitive attributes. Secondly, we show that violations of differential fairness is a problem of finding correlations between sensitive attributes and  classifier's outcomes. Therefore, \textbf{mdfa} searches for instances of violation of multi-differential fairness by predicting where in the features space the binary values of sensitive attributes and classifier's outcomes coincide. Lastly, worst-case violations are extracted by incrementally removing the least unfair instances.

This paper makes the following contributions:
\begin{itemize}
	\item The proposed  multi-differential fairness framework is the first attempt to identify disparate treatment for groups of individuals as small as computationally possible. 
	\item  \textbf{mdfa} efficiently identifies the individuals who are the most severely discriminated against by  black-box classifiers.
	\item We apply \textbf{mdfa} to a case study of a recidivism risk assessment in Broward County, Florida and find a sub-population of African-American defendants who are three times more likely to be considered at high risk of violent recidivism than similar individuals of other races.
	\item We apply \textbf{mdfa} to three other datasets related to crime, income and credit predictions and find that classifiers, even after being repaired for aggregate fairness, do discriminate against smaller sub-populations.
\end{itemize}

\paragraph{Related Work}
In the growing literature on algorithmic fairness (see  \cite{chouldechova2018frontiers} for a survey), this paper relates mostly to three axes of work. First, our paper, as in  \cite{hebert2017calibration,kim2018fairness,kearns2017preventing} provides a definition of fairness that protects group of individuals as small as computationally possible. Empirical observations in \cite{kearns2017preventing,dwork2012fairness} support defining fairness at the individual level since aggregate level fairness cannot protect specific sub-populations against severe discrimination.

Secondly, prior contributions on algorithmic disparate treatment \cite{zafar2017fairness}  have focused on whether sensitive attributes are used directly to train a classifier. This is a limitation when dealing with  classifiers  whose inputs are unknown. Differential fairness is inspired by differential privacy \cite{dwork2014algorithmic} and offers a general framework to measure whether a classifier exacerbates the biases encoded in the data. Reinterpretations of fairness as a privacy problem can be found in \cite{jagielski2018differentially,foulds2018intersectional}, but none of those contributions make a connection to disparate treatment.  Similar to prior work on disparate impact \cite{feldman2015certifying,chouldechova2017fair} there is a need  to re-balance the distribution of features conditioned on sensitive attributes. Our kernel matching technique deals with  covariate shift \cite{gretton2009covariate,cortes2008sample}, and has been used in domain adaptation (see e.g. \cite{mansour2009domain}) and counterfactual analysis (e.g. \cite{johansson2016learning}).

To the extent of our knowledge, there is no existing work on characterizing the most severe instances of algorithmic discrimination. However,  recent contributions offer approaches to bolster individual recourse. For example,  \cite{ustun2018actionable,russell2019efficient} develop algorithms to answer what-if questions. However, unlike \textbf{mdfa}, these approaches do not account for the fact that individuals with different sensitive attributes are not drawn from similar distributions.    

\section{Individual and Multi-Differential Fairness}


\paragraph{Preliminary}

An individual $i$ is defined by a tuple $((x_{i}, s_{i}), y_{i})$, where $x_{i}\in \mathcal{X}$ denotes $i$'s audited features; $s_{i}\in\mathcal{S}$ denotes the sensitive attributes; and $y_{i}\in \{-1, 1\}$ is a classifier $f$'s outcome. The auditor draws m samples $\{((x_{i}, s_{i}), y_{i})\}_{i=1}^{m}$ from a distribution $D$ on $\mathcal{X} \times \mathcal{S}\times \{-1, 1\}$. Features in $\mathcal{X}$ are not necessarily the ones used to train $f$. First, the auditor may not have access to all features used to train $f$. Secondly, the auditor may decide to deliberately leave out some features used to train $f$ out of $\mathcal{X}$ because those features should not be used to define similarity among individuals. 
For example, if $f$ classifies individuals according to their probability of repaying a loan, the auditor may consider that zipcode (correlated with race) should not be an auditing feature, although it was used to train $f$.  

\paragraph{Assumptions}
In our analysis we assume that the distributions of auditing features conditioned on sensitive attributes have common support.
\begin{assumption}
	\label{ass: 1}
	For all $x\in \mathcal{X}$, $Pr[S|X=x] > 0.$
\end{assumption}
 
\subsection{Individual Differential Fairness}
We define differential fairness as the guarantee that conditioned on features relevant to the tasks, a classifier's outcome is nearly independent of sensitive attributes: 

\begin{defn}(Individual Differential Fairness)
	\label{def: idf}
	For $\delta \geq 0$, a classifier $f$ is $\delta-$ differential fair if $\forall x\in\mathcal{X}, \forall s\in \mathcal{S}, \forall y\in\{-1, 1\}$
	\begin{small}
	\begin{equation}
	\label{eq: idf}
	e^{-\delta} \leq \frac{Pr[Y=y|S =s, x]}{Pr[Y=y|S\neq s, x]} \leq e^{\delta}
	\end{equation}
	\end{small}
\end{defn}

The parameter $\delta$ controls how much the distribution of the classifier's outcome $Y$ depends on sensitive attributes $S$ given that auditing feature is $x$; larger value of $\delta$ implies a less differentially fair classifier. 

\paragraph{Differential Fairness and Disparate Treatments}
A $\delta-$ differential fair classifier $f$ bounds the disparity of treatment between individuals with different sensitive attributes, the disparity being measured by the maximum divergence between the distributions $Pr(Y|S=s, x)$ and $Pr(Y|S\neq s, x)$:  
\begin{equation}
\nonumber
\begin{small}
\max_{y\in Y}\ln\left(\frac{Pr[Y|S=s, x]}{Pr[Y|S\neq s, x]}\right) \leq \delta
\end{small}
\end{equation}

\paragraph{Relation with Differential Privacy}
Differential fairness re-interprets disparate treatment as a differential privacy issue  \cite{dwork2014algorithmic} by bounding the leakage of sensitive attributes caused by $Y$ given what is already leaked by the auditing features $x$. Formally, the fairness condition \eqref{eq: idf} bounds the maximum divergence between the distributions $Pr(S|Y, x)$ and $Pr(S| x)$ by $\delta$. 

\paragraph{Individual Fairness}
Def. \eqref{def: idf} is an individual level definition of fairness, since it conditions the information leakage on auditing features $x$. Compared to the notion of individual fairness \cite{dwork2012fairness}, individual differential fairness does not require an explicit  similarity metric. This is a strength of our framework since  defining a similarity metric has been the main limitation of applying the concept of individual fairness  \cite{chouldechova2018frontiers}.   

\subsection{Multi-differential fairness}
Although useful, the notion of individual differential fairness cannot be computationally efficiently audited for. Looking for violations of individual differential fairness will require searching over a set of $2^{|\mathcal{X}|}$ individuals. Moreover,  a sample from a distribution over $\mathcal{X} \times \mathcal{S}\times \{-1, 1\}$ has a negligible probability to have two individuals with the same auditing features $x$ but different sensitive attributes $s$.

Therefore, we relax the definition of individual differential fairness and impose differential fairness for sub-populations. Formally, $\mathbb{C}$ denotes a collection of subsets or group of individuals $G$ in $\mathcal{X}$. The collection $\mathbb{C}_{\alpha}$ is $\alpha$-strong if for $G\in \mathbb{C}$ and $y\in \{-1, 1\}$, $Pr[Y=y \;\&\; x\in G] \geq \alpha$.  

\begin{defn}(Multi-Differential Fairness)
	\label{def: mdf}
	Consider a $\alpha$-strong collection $\mathbb{C}_{\alpha}$ of sub-populations of $\mathcal{X}$. For $0\leq \delta$, a classifier $f$ is $(\mathbb{C}_{\alpha}, \delta)$-multi differential fair with respect to $\mathcal{S}$ if $\forall s\in \mathcal{S}, \forall y\in\{-1, 1\}$ and $\forall G\in \mathbb{C}_{\alpha}$:
	\begin{equation}
	\begin{small}
	\label{eq: mdf}
	e^{-\delta} \leq \frac{Pr[Y=y|S=s, G]}{Pr[Y=y|S\neq s, G]} \leq e^{\delta} 
	\end{small}
	\end{equation}
\end{defn}

Multi-differential fairness guarantees that the outcome of a classifier $f$ is nearly mean-independent of protected attributes within any sub-population $G\in \mathbb{C}_{\alpha}$. The fairness condition in Eq. \ref{eq: mdf} applies only to subpopulations with $Pr[Y=y \;\&\; x\in G] \geq \alpha$ for $y\in\{-1, 1\}$. This is to avoid trivial cases where $\{x\in G \; \& \; Y=y\}$ is a singleton for some $y$, implying  $\delta=\infty$.

\paragraph{Collection of Indicators.}
We represent the collection of sub-populations $\mathbb{C}$ as a family of indicators: for $G\in \mathbb{C}$, there is an indicator $c: \mathcal{X}\rightarrow \{-1, 1\}$ such that $c(x)=1$ if and only if $x\in G$. The relaxation of differential fairness to a collection of groups or sub-population is akin to \cite{kim2018fairness,kearns2017preventing,hebert2017calibration}. $\mathbb{C}_{\alpha}$ is the computational bound on how granular our definition of fairness is. The richer $\mathbb{C}_{\alpha}$, the stronger the fairness guarantee offers by Def. \ref{def: mdf}. However, the complexity of $\mathbb{C}_{\alpha}$ is limited by the fact that we identify a sub-population $G$ via random samples drawn from a distribution over $\mathcal{X} \times \mathcal{S}\times \{-1, 1\}$. 

\section{Fairness Diagnostics: Worst-case violations}

The objective of \textbf{mdfa} is to find the sub-populations with the most severe violation of multi-differential fairness -- that is to solve for $s\in \mathcal{S}$ and $y\in \{-1, 1\}$
\begin{equation}
\label{eq: wvio}
\begin{small}
\sup_{S\in \mathbb{C}_{\alpha}}\ln\left(\frac{Pr[Y=1|S, S=s]}{Pr[Y=1|S, S\neq s]}\right).
\end{small}
\end{equation}
Our approach succeeds at tackling three challenges: (i) if $\mathbb{C}_{\alpha}$ is large, an auditing algorithm linearly dependent on $|\mathbb{C}|$ can be prohibitively expensive; (ii) the data needs to be balanced conditioned on sensitive attributes; (iii) finding efficiently the most-harmed sub-population implies that we can predict a function $c\in \mathbb{C}_{\alpha}$ for which we do not directly observe values $c(x)$. 

\subsection{Reduction to Agnostic Learning}
First, we reduce the problem of certifying for the lack of differential fairness to a agnostic learning problem. 

\paragraph{Multi Differential Fairness for Balanced Distribution}
In this section, we assume that for any $s\in \mathcal{S}$, the conditional distributions $p(x|S=s)$ and $p(x|S\neq s)$ are identical. This is not realistic for many datasets, but we will show how to handle unbalanced data in the next section. A balanced distribution does not leak any information on whether a sensitive attribute is equal to $s$: $Pr[S=s|x]=Pr[S\neq s|x]$. A violation of $(\mathbb{C}_{\alpha}, \delta)$- multi differential fairness simplifies then to  a sub-population $G\in \mathbb{C}_{\alpha}$, a $y\in \{-1, 1\}$ and $s\in \mathcal{S}$ such that
\begin{equation}
\begin{small}
\label{eq: unfair}
Pr[G, Y=y]\left\{Pr[S=s |G, Y=y] -\frac{1}{2}\right\}\geq \gamma,
\end{small}
\end{equation}
with $\gamma =\alpha \left(e^{\delta}/(1+e^{\delta})-1/2\right)$. $\gamma$ combines the size of the sub-population where a violation exists and the magnitude of the violation. We call a $\gamma-$ unfairness certificate any triple $(G, y, s)$ that satisfies Eq. \eqref{eq: unfair}. Further we postulate that $f$ is $\gamma-$unfair if and only if such certificate exists. Unfairness for balanced distributions is equivalent to the existence of sub-populations for which sensitive attributes can be predicted once the classifier's outcomes are observed.

Searching for $\gamma$-unfairness certificate reduces to mapping the auditing features $\{x_{i}\}$ to the labels $\{s_{i}y_{i}\}$.
\begin{lem}
	\label{lem: 1}
	Let $s\in \mathcal{S}$. Suppose that the data is balanced. $f$ is  $\gamma-$ multi-differential unfair for $y\in \{-1, 1\}$ if and only there exists $c\in \mathbb{C}_{\alpha}$ such that $Pr[rSY=c] \geq 1 - \rho(y) + 4\gamma$, where $r=sign(y)$ and $\rho(y)=Pr[S=rY]$. 
\end{lem}

Lemma \ref{lem: 1} allows us to reduce searching for a $(G, y, s)$ unfairness certificate to predicting where sensitive attribute and outcomes of $f$ (if $y=1$) or outcomes of $\neg f$ (if $y=-1$) coincide. Our proposed approach is to solve the following empirical loss minimization: 
\begin{equation}
\begin{small}
\label{eq: risk1}
\min_{c\in \mathbb{C}}\frac{1}{m}\displaystyle\sum_{i=1}^{m} l(c, a_{i}y_{i})   + Reg(c),
\end{small}
\end{equation}
where $l(.)$ is a $0-1$ loss function  and $Reg(.)$ a regularizer.

The following result shows that (i) our reduction to a learning problem leads to  an unbiased estimate of $\gamma$ ; (ii) there is a computational limit on how granular multi-differential fairness can be, since for many concept classes $\mathbb{C}$ agnostic learning is a NP-hard problem ( \cite{feldman2012agnostic}).  
\begin{thm}
	\label{thm: al}
	Let $\epsilon, \beta >0$ and $\mathbb{C}\subset 2^{\mathcal{X}}$. Let $\gamma^{'}\in (\gamma - \epsilon, \gamma + \epsilon)$. 
	\begin{enumerate}[(i)]
	    \item There exists an algorithm that by using $O(\log(|\mathcal{C}, \log(\frac{1}{\eta}), \frac{1}{\epsilon^{2}})$ samples $\{(x_{i}, s_{i}), y_{i}\}$ drawn from a balanced distribution $D$ outputs  with probability $1-\eta$  a $\gamma^{'}$-unfairness certificate  if $y_{i}$ are outcomes from a $\gamma-$unfair classifier;
	    \item $\mathbb{C}$ is agnostic learnable: there exists an algorithm that with $O(\log(|\mathcal{C}, \log(\frac{1}{\eta}), \frac{1}{\epsilon^{2}})$ samples $\{x_{i}, o_{i}\}$ drawn from a balanced distribution $D$ outputs with probability $1-\eta$, $Pr_{D}[h(x_{i})=o_{i}]  + \epsilon \geq max_{c\in \mathbb{C}}Pr_{D}[c(x_{i})=o_{i}]$
	\end{enumerate}
\end{thm}

\subsection{Unbalanced Data}
\paragraph{Imbalance Problem}
Multi-differential fairness measures the max-divergence distance between the posterior distribution $Pr(S|Y, x)$  and the prior one $Pr(S|x)$. Therefore, it requires knowledge of $Pr(S|x)$. In the previous section, we circumvent the issue by assuming $Pr[S=s|x]=1/2$. To generalize our approach, we propose to rebalance the data with the following weights: for $s\in \mathcal{S}$, $w_{s}(x,s)=Pr[S\neq s|x]/Pr[S=s|x]$ and for $s^{'}\neq s$, $w_{s}(x, s^{'}) =1$. Once reweighted, the conditional distributions $Pr_{w}(X|S=s)$ and $Pr_{w}(X|S\neq s)$ are identical and our learning reduction from the previous section applies.

However, in practice we do not have direct access to $w_{s}$. One approach is to directly estimate the density $P[S=s|x]$. This method is used in propensity-score matching methods (\cite{rosenbaum1983central}) in the context of counterfactual analysis. But, exact or estimated importance sampling results in large variance in finite sample (\cite{cortes2010learning}). Instead,  we use a kernel-based matching approach (\cite{gretton2009covariate} and \cite{cortes2008sample}). Our method considers real-value classification $h:\mathcal{X}\rightarrow \mathbb{R}$ such that $c(x)$ is equal to the sign of $h$\footnote{To prove our results, we will need $c(x)=g(h(x))$, where for small $\tau> 0$, $g(h)=sign(h)$ for $|h| > \tau/2$ and $g(h)=1$ with probability $h$ for $|h|<\tau/2$.}. The loss function $l(h, sy)$ in Eq. \eqref{eq: risk1} is assumed to be convex. Our setting includes, for example, support vector machine and logistic classification. The following result bounds above the change in the solution of Eq. \eqref{eq: risk1} when changing the weighting scheme from $u$ to $w$.

 \begin{lem}
 \label{lem: 3}
 Let $\phi$ be a feature mapping and $k$ be its associated kernel with $k(x, x^{'}) = \langle \phi(x), \phi(x^{'})\rangle$ and $||k||_{\infty} < \kappa < \infty$. Suppose that in Eq. \eqref{eq: risk1}, $Reg(h) = \lambda_{c} ||g||_{k}^{2}$ and that for $x\in \mathcal{X}$, $h(x)=\langle h|k(x,.)\rangle$. Suppose that $l$ is $\sigma-$ Lipchitz in its first argument. Denote $h_{u}$ and $h_{w}$ the solutions of the risk minimization Eq. \eqref{eq: risk1} with weights $u$ and $w$ respectively. Then,
 \begin{equation}
 \nonumber
 \begin{small}
  \forall x\in \mathcal{X}, |h_{u}(x) -h_{w}(x)| \leq \kappa^{2} \sigma \frac{\sqrt{cond(k)}}{\lambda_{c}} G_{k}(u, w),  
  \end{small}
 \end{equation}
 where $cond(k)$ is the condition number of the Gram matrix of $k$ and $G_{k}(u,w)$ is the maximum mean discrepancy between the distributions weighted by $u$ and $w$:
 \begin{equation}
     \nonumber
     \begin{small}
     G_{k}(u, w)= \left\lVert\displaystyle\sum_{i}u(x_{i})\phi(x_{i})-\displaystyle\sum_{i}w(x_{i})\phi(x_{i})\right\rVert.
     \end{small}
 \end{equation}
 \end{lem}

 Note that when the distribution is weighted with the importance sampling $w_{s}$, the maximum mean discrepancy of $Pr(X|S=s)$ and $Pr(X|S\neq s)$ is zero. By minimizing the maximum-mean discrepancy $G_{k}(u, w_{s})$, we minimize an upper bound on the pointwise difference between $h_{u}$ and $h_{w_{s}}$, that is the difference between the unfairness certificate we choose with weigths $u$ and the one we would have chosen if $Pr(X|S=s)=Pr(X|S\neq s)$. Therefore, \textbf{mdfa} solves:

\begin{equation}
\label{eq: risk2}
\begin{small}
\min_{\phi, u}\displaystyle\sum_{i}u_{i}(x_{i})l(\phi(x_{i}), a_{i}y_{i}) + Reg(c) + \widehat{G_{k}}(u, a)
\end{small}
\end{equation}
In our implementation, the feature representation $\phi$ is learned via a neural network that is then shared with both tasks of minimizing the re-weighted certifying risk and the empirical counterpart $\widehat{G_{k}}(u, s)$ of the maximum mean discrepancy between $Pr(X|S=s)$ and $Pr(X|S\neq s)$. The following result bounds above the bias in \textbf{mdfa}'s estimate of $\gamma$: 

\begin{thm}
	\label{thm: corr1}
    Let $\epsilon >0$ and $\eta\in(0,1)$. 
	Suppose that $\mathbb{C}$ is a concept class of VC dimension $d<\infty$. Solving for Eq. \eqref{eq: risk2} finds a $\gamma -\epsilon$-unfairness certificate if $f$ is $\gamma-$ unfair and at least $O\left(\frac{1}{\epsilon^{2}}\log(d) \log\left(\frac{1}{\eta}\right)\right)$ samples are queried. 
\end{thm}

\subsection{Worst-Case Violation}
Solving the empirical minimization Eq. \eqref{eq: risk2} allows certifying whether any black box classifier is multi-differential fair, but the solution of Eq. \eqref{eq: risk2} does not distinguish a large sub-population $S$ with low value of $\delta$ from a smaller sub-population with larger value of $\delta$. For example, consider two sub-populations of same size $G_{0}$ and $G_{\delta}$ for $\delta >0$. Assume that there is no violation of multi-differential fairness on $G_{0}$, but a $\delta-$ violation on $G_{\delta}$. The risk minimization Eq. \ref{eq: risk2} will pick indifferently $G_{\delta}$ and $G_{\delta}\cup G_{0}$ as unfairness certificates, although $G_{0}$ mixes the violation $G_{\delta}$ with a sub-population without any violation of differential fairness.

\paragraph{Worst-Case Violation Algorithm (WVA)}
At issue in the previous example is that for the sub-population $G_{0}$, choosing $c=1$ or $c=-1$ will lead to the same empirical risk Eq. \ref{eq: risk2}. To force $c(x)=-1$ for $x\in G_{0}$,  our approach is to put a slightly larger weight on samples whenever $s_{i}\neq y_{i}$. Now the empirical risk is smaller for $c=-1$ wherever there is no violation of multi differential fairness.  More generally, our worst-violation algorithm \ref{algo: 3} iteratively increases by $1 + \xi t$ the weight on samples whenever $s_{i}\neq y_{i}$, where $\xi > 0$.  At iteration $t$, the solution $c_{t}$ of the empirical risk minimization \eqref{eq: risk1} identifies a sub-population $G_{t}=\{x| c_{t}(x)=1\}$ with a $\delta(c_{t})-$ violation of differential fairness, with $\delta \geq \ln((1 - h(\xi t))/h(\xi t))$, where $h$ is an increasing function. The Algorithm \ref{algo: 3} terminates whenever either $|G_{t}|\leq \alpha$. At the second to the last iteration $T$, theorem \ref{thm: algo3_ana} guarantees that Algorithm \ref{algo: 3} will identify a sub-population $G_{T}$ with a $\delta_{T}$-multi differential fairness violation and  $\delta_{T}$ asymptotically close to $\delta_{m}$.

\begin{thm}
	\label{thm: algo3_ana}
	Suppose $\xi > 0, \epsilon >0, \eta\in (0, 1)$ and $\mathbb{C}\subset 2^{\mathcal{X}}$ is $\alpha-$strong. Suppose that the classifier $f$ has been certified with $\gamma$-multi-differential unfairness for $y\in\{-1, 1\}$. Denote $\delta_{m}$ the worst-case violation of multi differential fairness for $\mathbb{C}$ as defined in \eqref{eq: wvio}. With probability $1-\eta$, with $O\left(\frac{1}{\epsilon^{2}}\log(|C|) \log\left(\frac{4}{\eta}\right)\right)$ samples and after $O\left(\frac{4(\gamma+\alpha)}{2\gamma + 3\alpha}\frac{2(4\gamma - 2\rho(y) + 1)}{\xi}\right)$ iterations, Algorithm \ref{algo: 3} learns $c\in \mathbb{C}$ such that 
	\begin{equation}
	\left|\ln\left(\frac{Pr[Y=y|S=s, c(x)=1]}{Pr[Y=y|S\neq s, c(x)=1]} \right) - \delta_{m}\right| \leq \epsilon.
	\end{equation}
\end{thm}

\begin{algorithm}[t]
	\caption{Worst Violation Algorithm (WVA)}
	\label{algo: 3}
	\textbf{Input:}  $\{((x_{i}, a_{i}), y_{i})\}_{i=1}^{m}$, $\mathbb{C}\subset 2^{|\mathcal{X}|}$, $\xi$, $\alpha$, weights $u$, $y\in \{-1, 1\}$
	\begin{algorithmic}[1]
		\STATE  $\alpha_{0} =1$, $u_{it}=u$ 
		\WHILE {$\alpha_{t} > \alpha$}  
		\STATE $c_{t} = argmin_{c\in \mathbb{C}}\frac{1}{m}\displaystyle\sum_{i=1}^{m} u_{it}(x)l(a_{i}y_{i}, c(x_{i})) + \lambda_{c}Reg(c)$
		
		\STATE $\left. \hat{\delta}_{t}\gets \displaystyle\sum_{\substack{i=1,  c(x_{i})=1 \\ y_{i}=y, a_{i}=1}}^{m} u_{i}(x_{i})\middle/ \displaystyle\sum_{\substack{i=1,  c(x_{i})=1 \\ y_{i}=y, a_{i}=-1}}^{m} u_{i}(x_{i})\right.$
		
		\STATE $\left. \hat{\alpha_{t}} \gets\gets \displaystyle\sum_{\substack{i=1,  c(x_{i})=1, y_{i}=1}}^{m} u_{i}(x_{i})\middle / \displaystyle\sum_{i=1}^{m} u_{i}(x_{i})\right.$
		
		\STATE $t\gets t +1$, $u_{it}\gets u_{it} + u_{i}\xi$ if $a_{i}\neq y_{i}$ and $y_{i}=y$.
		
		\ENDWHILE   
		\STATE{\bfseries Return} $\ln(\delta_{t})$.
	\end{algorithmic}
\end{algorithm}

\subsection{\textbf{mdfa} Auditor}
Putting the building blocks together allows us to design a fairness diagnostic tool  \textbf{mdfa} that identifies efficiently the most severe violation of differential unfairness.  

\paragraph{Architecture}
Inputs are a dataset with a classifier's outcomes (labels $\pm 1$) along with auditing features. \textbf{mfda} first uses a neural network with four fully connected layers of $8$ neurons to express the weights $u$ as a function of the features $x$ and minimizes the maximum-mean discrepancy function $\hat{G_{k}}(u, s)$. The outputs of the last hidden layer in the neural network are used as a feature mapping $\phi$ and serve along the estimated weights $u$ as an input to the empirical minimization Eq. \eqref{eq: risk2}, which outputs a certificate $(c, y, s)\in \mathbb{C}\times \{-1, 1\}\times \mathcal{S}$ of unfairness. The weights are re-adjusted until  the identified worst-case violation has a size smaller than $\alpha$. When terminating, \textbf{mdfa} outputs an estimate of the most-harmed sub-population $c_{m}$ along with an estimate of $\delta_{m}$.

\paragraph{Cross-Validation}
The auditor chooses the minimum size $\alpha$ of the worst-case violation they would like to identify.  The advantage of our approach is that, although we do not have ground truth for unfair treatment, we can propose heuristics to cross-validate our choice of regularization parameters used in Eq. \eqref{eq: risk2}. First, we split $70\%/30\%$ the input data into a train and test set. Using a $5-$fold cross-validation, \textbf{mdfa} is trained on four folds and a grid search looks for regularization parameters that minimize the maximum-mean-discrepancy $\hat{G_{k}}(u, s)$ and the empirical risk on the fifth fold. Once \textbf{mdfa} is trained, the estimated $\delta_{m}$ and the corresponding characteristics of the most-harmed sub-population are computed on the test set.  

\section{Experimental Results}

\subsection{Synthetic Data}
A synthetic data is constructed by drawing independently two features $X_{1}$ and $X_{2}$ from two normal distributions $N(0, 1)$. We consider a binary protected attribute  $\mathcal{S}=\{-1, 1\}$ drawn from a Bernouilli distribution with $S=1$ with probability $w(x)=\frac{e^{\mu * (x_{1}- x_{2}))^{2}}}{1 + e^{\mu * (x_{1}+ x_{2}))^{2}}}$. $\mu$ is the imbalance factor. $\mu=0$ means that the data is perfectly balanced. 

The data is labeled according to the sign of $(X_{1} + X_{2} + e) ^{3}$, where is $e$ is a noise drawn from $N(0, 0.2)$. The audited classifier $f$ is a logistic regression classifier that is altered to generate instances of differential unfairness. For $x^{2}_{1} + x^{2}_{2} \leq 1$, if $S=-1$, the classifier's outcomes $Y$ is changed from $-1$ to $1$ with probability $1 -\nu\in (0, 1]$; if $S=1$, all $Y=-1$ are changed to $Y=1$. For $\nu=0$, the audited classifier is differentially fair; however, as $\nu$ increases, in the half circle $\{(x_{1}, x_{2})|x^{2}_{1} + x^{2}_{2} \leq 1 \mbox{ and } y=-1\}$ there is a fraction $\nu$ of individuals with $S=1$ who are not treated similarly as individuals with $S=-1$. 
\paragraph{Results}
First, we test whether Algorithm \ref{algo: 3} identifies correctly the worst-case violation that occurs in the sub-space $\{(x_{1}, x_{2})|x^{2}_{1} + x^{2}_{2} \leq 1 \mbox{ and } y=-1\}$. \textbf{mdfa} is trained using a support vector machine (RBF kernel) on a unbalanced data ($\mu=0.2$) with value of $\delta_{m}$ varying from $0$ to $3.0$. Figure \ref{fig: 1a} plots the estimated $\hat{\delta_{m}}$ against the true one $\delta_{m}$ and shows that \textbf{mdfa}'s estimate $\hat{\delta_{m}}$ is unbiased.  Figure \ref{fig: 1b} shows that at each iteration of the Algorithm \ref{algo: 3}, the estimated worst-case violation $\hat{\delta}_{m}$ progresses toward the true value $\delta_{m}$. Secondly,  we compare our balancing approach $MMD$ to alternative re-balancing approaches: (i) uniform weights ($UW$) with $u(x)=1/m$ for all $x$ and (ii) importance sampling ($IS$)  with exact weights $w(x)$.  $UW$ applies \textbf{mdfa} without rebalancing. $IS$ uses oracle access to the importance sampling weights $w$, since they are known in this synthetic experiment. Figure \ref{fig: 1c} plots the bias $\hat{\gamma} - \gamma$ for each unfairness certificate obtained by \textbf{mdfa} with varying values of the imbalance factor $\mu$. It shows that minimizing the maximum-mean discrepancy function is the only method that generates unbiased certificates regardless of  data imbalance. Bias in estimates obtained with $UW$ confirms that absent of a re-weighting scheme, \textbf{mdfa} cannot disentangle the information related to $S$ leaked by the features $x$ from the one leaked by the classifier's outcomes $y$.   Using importance sampling weights  ($IS$) directly does not perform well: this confirms previous observations in the literature that in finite sample, the variance of the importance sample weights can be detrimental to a re-balancing approach. Lastly, Figure \ref{fig: 1d} shows that \textbf{mdfa}'s estimates of $\delta_{m}$ are robust to diverse classes of classifiers, including support vector machines with non-linear $SVM-RBF$ or linear $SVM-Lin$ kernels and random forest $RF$.

	\begin{figure}[t!]
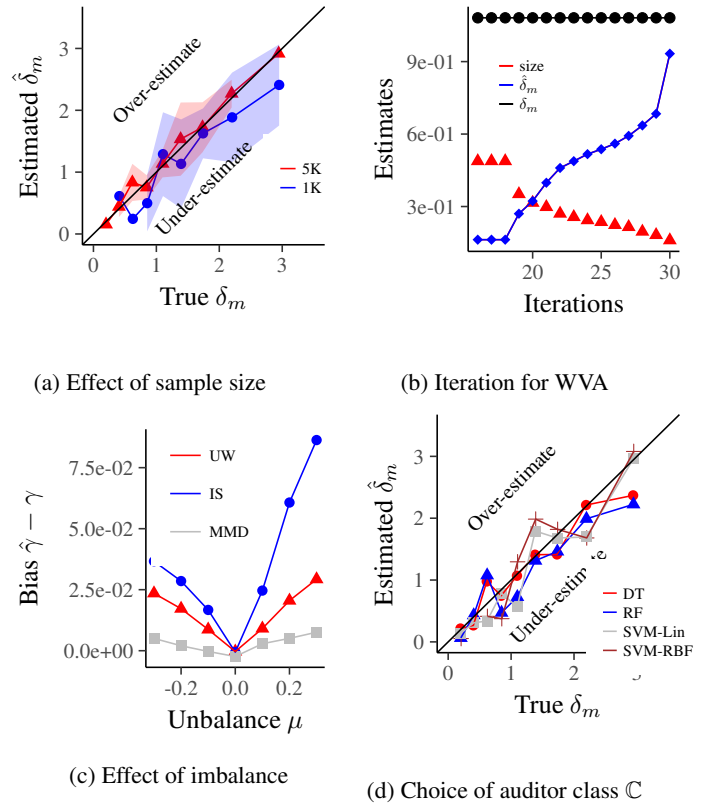

		
	    	\begin{subfigure}{0.45\linewidth}
	
\begin{tikzpicture}[x=1pt,y=1pt]
\definecolor{fillColor}{RGB}{255,255,255}
\path[use as bounding box,fill=fillColor,fill opacity=0.00] (0,0) rectangle (126.47,126.47);
\begin{scope}
\path[clip] (  0.00,  0.00) rectangle (126.47,126.47);
\definecolor{drawColor}{RGB}{255,255,255}
\definecolor{fillColor}{RGB}{255,255,255}

\path[draw=drawColor,line width= 0.6pt,line join=round,line cap=round,fill=fillColor] (  0.00,  0.00) rectangle (126.47,126.47);
\end{scope}
\begin{scope}
\path[clip] ( 29.34, 30.85) rectangle (120.97,120.97);
\definecolor{fillColor}{RGB}{255,0,0}

\path[fill=fillColor] ( 38.26, 41.61) --
	( 40.90, 37.03) --
	( 35.62, 37.03) --
	cycle;

\path[fill=fillColor] ( 43.13, 48.31) --
	( 45.78, 43.74) --
	( 40.49, 43.74) --
	cycle;

\path[fill=fillColor] ( 48.22, 57.51) --
	( 50.86, 52.93) --
	( 45.58, 52.93) --
	cycle;

\path[fill=fillColor] ( 53.65, 55.63) --
	( 56.29, 51.06) --
	( 51.01, 51.06) --
	cycle;

\path[fill=fillColor] ( 59.63, 64.61) --
	( 62.28, 60.03) --
	( 56.99, 60.03) --
	cycle;

\path[fill=fillColor] ( 66.48, 73.89) --
	( 69.12, 69.32) --
	( 63.84, 69.32) --
	cycle;

\path[fill=fillColor] ( 74.77, 78.38) --
	( 77.41, 73.81) --
	( 72.13, 73.81) --
	cycle;

\path[fill=fillColor] ( 85.78, 91.13) --
	( 88.42, 86.55) --
	( 83.14, 86.55) --
	cycle;

\path[fill=fillColor] (103.57,106.29) --
	(106.21,101.71) --
	(100.92,101.71) --
	cycle;
\definecolor{fillColor}{RGB}{0,0,255}

\path[fill=fillColor] ( 43.31, 49.28) circle (  1.96);

\path[fill=fillColor] ( 48.40, 40.66) circle (  1.96);

\path[fill=fillColor] ( 53.83, 46.60) circle (  1.96);

\path[fill=fillColor] ( 59.81, 65.15) circle (  1.96);

\path[fill=fillColor] ( 66.66, 61.47) circle (  1.96);

\path[fill=fillColor] ( 74.95, 73.02) circle (  1.96);

\path[fill=fillColor] ( 85.96, 79.08) circle (  1.96);

\path[fill=fillColor] (103.75, 91.36) circle (  1.96);
\definecolor{drawColor}{RGB}{0,0,255}

\path[draw=drawColor,line width= 0.6pt,line join=round] ( 43.31, 49.28) --
	( 48.40, 40.66) --
	( 53.83, 46.60) --
	( 59.81, 65.15) --
	( 66.66, 61.47) --
	( 74.95, 73.02) --
	( 85.96, 79.08) --
	(103.75, 91.36);
\definecolor{drawColor}{RGB}{255,0,0}

\path[draw=drawColor,line width= 0.6pt,line join=round] ( 38.26, 38.56) --
	( 43.13, 45.26) --
	( 48.22, 54.45) --
	( 53.65, 52.58) --
	( 59.63, 61.56) --
	( 66.48, 70.84) --
	( 74.77, 75.33) --
	( 85.78, 88.07) --
	(103.57,103.24);
\definecolor{fillColor}{RGB}{0,0,255}

\path[fill=fillColor,fill opacity=0.20] ( 53.83, 57.30) --
	( 59.81, 81.02) --
	( 66.66, 78.31) --
	( 74.95, 82.47) --
	( 85.96, 96.03) --
	(103.75,106.72) --
	(103.75, 75.99) --
	( 85.96, 62.13) --
	( 74.95, 63.57) --
	( 66.66, 44.64) --
	( 59.81, 49.28) --
	( 53.83, 35.91) --
	cycle;
\definecolor{fillColor}{RGB}{255,0,0}

\path[fill=fillColor,fill opacity=0.20] ( 43.13, 49.20) --
	( 48.22, 61.49) --
	( 53.65, 57.17) --
	( 59.63, 66.77) --
	( 66.48, 84.71) --
	( 74.77, 84.80) --
	( 85.78, 93.32) --
	( 85.78, 82.83) --
	( 74.77, 65.87) --
	( 66.48, 56.97) --
	( 59.63, 56.35) --
	( 53.65, 48.00) --
	( 48.22, 47.42) --
	( 43.13, 41.33) --
	cycle;
\definecolor{drawColor}{RGB}{0,0,0}

\path[draw=drawColor,line width= 0.6pt,line join=round] ( 29.34, 30.85) -- (120.97,120.97);

\node[text=drawColor,rotate= 45.00,anchor=base,inner sep=0pt, outer sep=0pt, scale=  0.78] at ( 59.21, 91.56) {Over-estimate};

\node[text=drawColor,rotate= 45.00,anchor=base,inner sep=0pt, outer sep=0pt, scale=  0.78] at ( 77.06, 50.59) {Under-estimate};
\end{scope}
\begin{scope}
\path[clip] (  0.00,  0.00) rectangle (126.47,126.47);
\definecolor{drawColor}{RGB}{0,0,0}

\path[draw=drawColor,line width= 0.6pt,line join=round] ( 29.34, 30.85) --
	( 29.34,120.97);
\end{scope}
\begin{scope}
\path[clip] (  0.00,  0.00) rectangle (126.47,126.47);
\definecolor{drawColor}{gray}{0.30}

\node[text=drawColor,anchor=base east,inner sep=0pt, outer sep=0pt, scale=  0.80] at ( 24.39, 32.19) {0};

\node[text=drawColor,anchor=base east,inner sep=0pt, outer sep=0pt, scale=  0.80] at ( 24.39, 55.60) {1};

\node[text=drawColor,anchor=base east,inner sep=0pt, outer sep=0pt, scale=  0.80] at ( 24.39, 79.01) {2};

\node[text=drawColor,anchor=base east,inner sep=0pt, outer sep=0pt, scale=  0.80] at ( 24.39,102.42) {3};
\end{scope}
\begin{scope}
\path[clip] (  0.00,  0.00) rectangle (126.47,126.47);
\definecolor{drawColor}{gray}{0.20}

\path[draw=drawColor,line width= 0.6pt,line join=round] ( 26.59, 34.94) --
	( 29.34, 34.94);

\path[draw=drawColor,line width= 0.6pt,line join=round] ( 26.59, 58.35) --
	( 29.34, 58.35);

\path[draw=drawColor,line width= 0.6pt,line join=round] ( 26.59, 81.76) --
	( 29.34, 81.76);

\path[draw=drawColor,line width= 0.6pt,line join=round] ( 26.59,105.17) --
	( 29.34,105.17);
\end{scope}
\begin{scope}
\path[clip] (  0.00,  0.00) rectangle (126.47,126.47);
\definecolor{drawColor}{RGB}{0,0,0}

\path[draw=drawColor,line width= 0.6pt,line join=round] ( 29.34, 30.85) --
	(120.97, 30.85);
\end{scope}
\begin{scope}
\path[clip] (  0.00,  0.00) rectangle (126.47,126.47);
\definecolor{drawColor}{gray}{0.20}

\path[draw=drawColor,line width= 0.6pt,line join=round] ( 33.50, 28.10) --
	( 33.50, 30.85);

\path[draw=drawColor,line width= 0.6pt,line join=round] ( 57.30, 28.10) --
	( 57.30, 30.85);

\path[draw=drawColor,line width= 0.6pt,line join=round] ( 81.10, 28.10) --
	( 81.10, 30.85);

\path[draw=drawColor,line width= 0.6pt,line join=round] (104.91, 28.10) --
	(104.91, 30.85);
\end{scope}
\begin{scope}
\path[clip] (  0.00,  0.00) rectangle (126.47,126.47);
\definecolor{drawColor}{gray}{0.30}

\node[text=drawColor,anchor=base,inner sep=0pt, outer sep=0pt, scale=  0.80] at ( 33.50, 20.39) {0};

\node[text=drawColor,anchor=base,inner sep=0pt, outer sep=0pt, scale=  0.80] at ( 57.30, 20.39) {1};

\node[text=drawColor,anchor=base,inner sep=0pt, outer sep=0pt, scale=  0.80] at ( 81.10, 20.39) {2};

\node[text=drawColor,anchor=base,inner sep=0pt, outer sep=0pt, scale=  0.80] at (104.91, 20.39) {3};
\end{scope}
\begin{scope}
\path[clip] (  0.00,  0.00) rectangle (126.47,126.47);
\definecolor{drawColor}{RGB}{0,0,0}

\node[text=drawColor,anchor=base,inner sep=0pt, outer sep=0pt, scale=  1.00] at ( 75.15,  8.00) {True $\delta_{m}$};
\end{scope}
\begin{scope}
\path[clip] (  0.00,  0.00) rectangle (126.47,126.47);
\definecolor{drawColor}{RGB}{0,0,0}

\node[text=drawColor,rotate= 90.00,anchor=base,inner sep=0pt, outer sep=0pt, scale=  1.00] at ( 12.39, 75.91) {Estimated $\hat{\delta}_{m}$};
\end{scope}
\begin{scope}
\path[clip] (  0.00,  0.00) rectangle (126.47,126.47);
\definecolor{fillColor}{RGB}{255,255,255}

\path[fill=fillColor] ( 97.77, 43.16) rectangle (125.85, 72.61);
\end{scope}
\begin{scope}
\path[clip] (  0.00,  0.00) rectangle (126.47,126.47);
\definecolor{drawColor}{RGB}{255,0,0}

\path[draw=drawColor,line width= 0.6pt,line join=round] (104.18, 59.69) -- (109.96, 59.69);
\end{scope}
\begin{scope}
\path[clip] (  0.00,  0.00) rectangle (126.47,126.47);
\definecolor{drawColor}{RGB}{0,0,255}

\path[draw=drawColor,line width= 0.6pt,line join=round] (104.18, 52.46) -- (109.96, 52.46);
\end{scope}
\begin{scope}
\path[clip] (  0.00,  0.00) rectangle (126.47,126.47);
\definecolor{drawColor}{RGB}{0,0,0}

\node[text=drawColor,anchor=base west,inner sep=0pt, outer sep=0pt, scale=  0.60] at (112.49, 57.62) {5K};
\end{scope}
\begin{scope}
\path[clip] (  0.00,  0.00) rectangle (126.47,126.47);
\definecolor{drawColor}{RGB}{0,0,0}

\node[text=drawColor,anchor=base west,inner sep=0pt, outer sep=0pt, scale=  0.60] at (112.49, 50.40) {1K};
\end{scope}
\end{tikzpicture}
	
		\caption{Effect of sample size}
		\label{fig: 1a}
		\end{subfigure}
		\begin{subfigure}{0.45\linewidth}
			\input{figure1d.tex}
			\caption{Iteration for WVA}
			\label{fig: 1b}
		\end{subfigure}
			\begin{subfigure}{0.54\linewidth}
\begin{tikzpicture}[x=1pt,y=1pt]
\definecolor{fillColor}{RGB}{255,255,255}
\path[use as bounding box,fill=fillColor,fill opacity=0.00] (0,0) rectangle (126.47,126.47);
\begin{scope}
\path[clip] (  0.00,  0.00) rectangle (126.47,126.47);
\definecolor{drawColor}{RGB}{255,255,255}
\definecolor{fillColor}{RGB}{255,255,255}

\path[draw=drawColor,line width= 0.6pt,line join=round,line cap=round,fill=fillColor] ( -0.00,  0.00) rectangle (126.47,126.47);
\end{scope}
\begin{scope}
\path[clip] ( 53.33, 30.85) rectangle (120.97,120.97);
\definecolor{fillColor}{RGB}{255,0,0}

\path[fill=fillColor] ( 56.40, 61.91) --
	( 59.05, 57.34) --
	( 53.76, 57.34) --
	cycle;

\path[fill=fillColor] ( 66.65, 56.10) --
	( 69.30, 51.52) --
	( 64.01, 51.52) --
	cycle;

\path[fill=fillColor] ( 76.90, 48.23) --
	( 79.54, 43.65) --
	( 74.26, 43.65) --
	cycle;

\path[fill=fillColor] ( 87.15, 39.88) --
	( 89.79, 35.31) --
	( 84.51, 35.31) --
	cycle;

\path[fill=fillColor] ( 97.40, 48.59) --
	(100.04, 44.01) --
	( 94.76, 44.01) --
	cycle;

\path[fill=fillColor] (107.65, 59.16) --
	(110.29, 54.58) --
	(105.01, 54.58) --
	cycle;

\path[fill=fillColor] (117.90, 67.24) --
	(120.54, 62.67) --
	(115.26, 62.67) --
	cycle;
\definecolor{fillColor}{RGB}{0,0,255}

\path[fill=fillColor] ( 56.40, 71.00) circle (  1.96);

\path[fill=fillColor] ( 66.65, 63.50) circle (  1.96);

\path[fill=fillColor] ( 76.90, 52.57) circle (  1.96);

\path[fill=fillColor] ( 87.15, 36.59) circle (  1.96);

\path[fill=fillColor] ( 97.40, 59.90) circle (  1.96);

\path[fill=fillColor] (107.65, 93.23) circle (  1.96);

\path[fill=fillColor] (117.90,116.88) circle (  1.96);
\definecolor{fillColor}{RGB}{190,190,190}

\path[fill=fillColor] ( 54.44, 39.83) --
	( 58.37, 39.83) --
	( 58.37, 43.75) --
	( 54.44, 43.75) --
	cycle;

\path[fill=fillColor] ( 64.69, 37.05) --
	( 68.62, 37.05) --
	( 68.62, 40.97) --
	( 64.69, 40.97) --
	cycle;

\path[fill=fillColor] ( 74.94, 34.90) --
	( 78.86, 34.90) --
	( 78.86, 38.83) --
	( 74.94, 38.83) --
	cycle;

\path[fill=fillColor] ( 85.19, 32.98) --
	( 89.11, 32.98) --
	( 89.11, 36.91) --
	( 85.19, 36.91) --
	cycle;

\path[fill=fillColor] ( 95.44, 37.82) --
	( 99.36, 37.82) --
	( 99.36, 41.75) --
	( 95.44, 41.75) --
	cycle;

\path[fill=fillColor] (105.69, 39.85) --
	(109.61, 39.85) --
	(109.61, 43.77) --
	(105.69, 43.77) --
	cycle;

\path[fill=fillColor] (115.94, 42.17) --
	(119.86, 42.17) --
	(119.86, 46.10) --
	(115.94, 46.10) --
	cycle;
\definecolor{drawColor}{RGB}{0,0,255}

\path[draw=drawColor,line width= 0.6pt,line join=round] ( 56.40, 71.00) --
	( 66.65, 63.50) --
	( 76.90, 52.57) --
	( 87.15, 36.59) --
	( 97.40, 59.90) --
	(107.65, 93.23) --
	(117.90,116.88);
\definecolor{drawColor}{RGB}{255,0,0}

\path[draw=drawColor,line width= 0.6pt,line join=round] ( 56.40, 58.86) --
	( 66.65, 53.04) --
	( 76.90, 45.18) --
	( 87.15, 36.83) --
	( 97.40, 45.54) --
	(107.65, 56.10) --
	(117.90, 64.19);
\definecolor{drawColor}{RGB}{190,190,190}

\path[draw=drawColor,line width= 0.6pt,line join=round] ( 56.40, 41.79) --
	( 66.65, 39.01) --
	( 76.90, 36.86) --
	( 87.15, 34.94) --
	( 97.40, 39.78) --
	(107.65, 41.81) --
	(117.90, 44.14);
\end{scope}
\begin{scope}
\path[clip] (  0.00,  0.00) rectangle (126.47,126.47);
\definecolor{drawColor}{RGB}{0,0,0}

\path[draw=drawColor,line width= 0.6pt,line join=round] ( 53.33, 30.85) --
	( 53.33,120.97);
\end{scope}
\begin{scope}
\path[clip] (  0.00,  0.00) rectangle (126.47,126.47);
\definecolor{drawColor}{gray}{0.30}

\node[text=drawColor,anchor=base east,inner sep=0pt, outer sep=0pt, scale=  0.80] at ( 48.38, 34.38) {0.0e+00};

\node[text=drawColor,anchor=base east,inner sep=0pt, outer sep=0pt, scale=  0.80] at ( 48.38, 57.48) {2.5e-02};

\node[text=drawColor,anchor=base east,inner sep=0pt, outer sep=0pt, scale=  0.80] at ( 48.38, 80.58) {5.0e-02};

\node[text=drawColor,anchor=base east,inner sep=0pt, outer sep=0pt, scale=  0.80] at ( 48.38,103.68) {7.5e-02};
\end{scope}
\begin{scope}
\path[clip] (  0.00,  0.00) rectangle (126.47,126.47);
\definecolor{drawColor}{gray}{0.20}

\path[draw=drawColor,line width= 0.6pt,line join=round] ( 50.58, 37.14) --
	( 53.33, 37.14);

\path[draw=drawColor,line width= 0.6pt,line join=round] ( 50.58, 60.24) --
	( 53.33, 60.24);

\path[draw=drawColor,line width= 0.6pt,line join=round] ( 50.58, 83.33) --
	( 53.33, 83.33);

\path[draw=drawColor,line width= 0.6pt,line join=round] ( 50.58,106.43) --
	( 53.33,106.43);
\end{scope}
\begin{scope}
\path[clip] (  0.00,  0.00) rectangle (126.47,126.47);
\definecolor{drawColor}{RGB}{0,0,0}

\path[draw=drawColor,line width= 0.6pt,line join=round] ( 53.33, 30.85) --
	(120.97, 30.85);
\end{scope}
\begin{scope}
\path[clip] (  0.00,  0.00) rectangle (126.47,126.47);
\definecolor{drawColor}{gray}{0.20}

\path[draw=drawColor,line width= 0.6pt,line join=round] ( 66.65, 28.10) --
	( 66.65, 30.85);

\path[draw=drawColor,line width= 0.6pt,line join=round] ( 87.15, 28.10) --
	( 87.15, 30.85);

\path[draw=drawColor,line width= 0.6pt,line join=round] (107.65, 28.10) --
	(107.65, 30.85);
\end{scope}
\begin{scope}
\path[clip] (  0.00,  0.00) rectangle (126.47,126.47);
\definecolor{drawColor}{gray}{0.30}

\node[text=drawColor,anchor=base,inner sep=0pt, outer sep=0pt, scale=  0.80] at ( 66.65, 20.39) {-0.2};

\node[text=drawColor,anchor=base,inner sep=0pt, outer sep=0pt, scale=  0.80] at ( 87.15, 20.39) {0.0};

\node[text=drawColor,anchor=base,inner sep=0pt, outer sep=0pt, scale=  0.80] at (107.65, 20.39) {0.2};
\end{scope}
\begin{scope}
\path[clip] (  0.00,  0.00) rectangle (126.47,126.47);
\definecolor{drawColor}{RGB}{0,0,0}

\node[text=drawColor,anchor=base,inner sep=0pt, outer sep=0pt, scale=  1.00] at ( 87.15,  8.00) {Unbalance  $\mu$};
\end{scope}
\begin{scope}
\path[clip] (  0.00,  0.00) rectangle (126.47,126.47);
\definecolor{drawColor}{RGB}{0,0,0}

\node[text=drawColor,rotate= 90.00,anchor=base,inner sep=0pt, outer sep=0pt, scale=  1.00] at ( 12.39, 75.91) {Bias $\hat{\gamma}-\gamma$};
\end{scope}
\begin{scope}
\path[clip] (  0.00,  0.00) rectangle (126.47,126.47);
\definecolor{fillColor}{RGB}{255,255,255}

\path[fill=fillColor] ( 55.39, 69.26) rectangle ( 98.62,127.62);
\end{scope}
\begin{scope}
\path[clip] (  0.00,  0.00) rectangle (126.47,126.47);
\definecolor{drawColor}{RGB}{255,0,0}

\path[draw=drawColor,line width= 0.6pt,line join=round] ( 62.53,111.09) -- ( 74.09,111.09);
\end{scope}
\begin{scope}
\path[clip] (  0.00,  0.00) rectangle (126.47,126.47);
\definecolor{drawColor}{RGB}{0,0,255}

\path[draw=drawColor,line width= 0.6pt,line join=round] ( 62.53, 96.63) -- ( 74.09, 96.63);
\end{scope}
\begin{scope}
\path[clip] (  0.00,  0.00) rectangle (126.47,126.47);
\definecolor{drawColor}{RGB}{190,190,190}

\path[draw=drawColor,line width= 0.6pt,line join=round] ( 62.53, 82.18) -- ( 74.09, 82.18);
\end{scope}
\begin{scope}
\path[clip] (  0.00,  0.00) rectangle (126.47,126.47);
\definecolor{drawColor}{RGB}{0,0,0}

\node[text=drawColor,anchor=base west,inner sep=0pt, outer sep=0pt, scale=  0.60] at ( 77.35,109.02) {UW};
\end{scope}
\begin{scope}
\path[clip] (  0.00,  0.00) rectangle (126.47,126.47);
\definecolor{drawColor}{RGB}{0,0,0}

\node[text=drawColor,anchor=base west,inner sep=0pt, outer sep=0pt, scale=  0.60] at ( 77.35, 94.57) {IS};
\end{scope}
\begin{scope}
\path[clip] (  0.00,  0.00) rectangle (126.47,126.47);
\definecolor{drawColor}{RGB}{0,0,0}

\node[text=drawColor,anchor=base west,inner sep=0pt, outer sep=0pt, scale=  0.60] at ( 77.35, 80.11) {MMD};
\end{scope}
\end{tikzpicture}
				\caption{Effect of imbalance}
				\label{fig: 1c}
			\end{subfigure}
			\begin{subfigure}{0.45\linewidth}
				\input{figure1b.tex}
				\caption{Choice of auditor class $\mathbb{C}$}
				\label{fig: 1d}
			\end{subfigure}
			\caption{Performance of \textbf{mdfa} on synthetic data. If not precised otherwise, \textbf{mdfa} is trained with a support vector machine with $5K$ samples and imbalance factor $\mu=0.2$}
	\end{figure}

\subsection{Case Study: COMPAS}
We  apply our method to the COMPAS algorithm, 
widely used to assess the likelihood of a defendant to become a recidivist (\cite{ProPublica2016}). The research question is whether without knowledge of the design of COMPAS, \textbf{mdfa} can identify group of individuals that could argue for a disparate treatment. The data collected by ProPublica in Broward County from 2013 to 2015  contains $7K$ individuals along with a risk score and a risk category assigned by COMPAS. We transform the risk category into a binary variable equal to $1$ for individuals assigned in the high risk category (risk score between $8$ and $10$). The data provides us with information related to the historical criminal history, misdemeanors, gender, age and race of each individual.  

\begin{table}[h!]
\begin{footnotesize}
\centering
	\begin{tabular}{l|l|l}
		Features & Race    & Gender  \\
		\hline
		\hline
		I, II, III, IV, V         & 0.023761 &   $-0.0001$        \\
		       & $\scriptsize{(0.004312)}$ &   $\scriptsize{(0.0006)}$        \\
		I, II, III         & 0.023599&     $-0.000037 $       \\
	        & $\scriptsize{( 0.00384)}$ & $\scriptsize{(0.0002)}$         \\
		I, II          & $0.027382 $ & $-0.00005$         \\
		 & $\scriptsize{(0.019383)}$ & $\scriptsize{(0.0001)}$         \\
	\end{tabular}
	\caption{Certifying the lack of differential fairness in COMPAS risk classification. Features are as follows: I: count of prior felonies; II: degree of current charge (criminal vs non-criminal); III: age; IV: count of juvenile prior felonies; V: count of juvenile prior misdemeanors. () indicates standard deviations.}
	\label{tab: 1}
\end{footnotesize}
\end{table}

\paragraph{Certifying the Lack of Differential Fairness}
We assess whether \textbf{mdfa} finds unfairness certificates by running only one iteration of algorithm \ref{algo: 3} on $100$ different $70/30\%$ train/test splits. The assessment is made for two binary sensitive attributes:  whether an individual self-identifies as Afro-American; and, whether an individual self-identifies as Male. Table \ref{tab: 1} reports the unfairness level $\gamma$ for each of those sub-groups: a value significantly larger than zero indicates the existence of sub-populations where similar individuals are treated differently by COMPAS.  In the first row of Table \ref{tab: 1}, we use prior felonies (I), degree of current charges (II), age (III), juvenile felonies (IV) and misdemeanors (V) as auditing features. We find a significant level of differential unfairness in the COMPAS risk classification ($\gamma =0.024\pm 004$) if the binary sensitive attribute is race. On the other hand, we do not find any evidence of differential unfairness when the sensitive attribute is gender. The results are robust to different choices for the auditing features, although standard deviations are higher when using only prior felonies (I) and degree of current charges (II).

\paragraph{Worst Violations}
We run \textbf{mdfa} on $100$ different $70/30\%$ train/test splits and report  average value of auditing features and recidivism risk for the whole population and the worst-case subpopulation in Table \ref{tab: 2}. The first two columns show that the distribution of features in the whole population is disperse and differs between African American (AA) and Other. 
This is due to the data imbalance issue (c.f. Section 3). The probability of being classified as high risk is $0.14$ for African-American, thereby $2.7$ times higher than for non-African American. However, it is unclear whether that difference could be explained either by the distribution imbalance or by the classifier's disparate treatment. The two last columns in Table \ref{tab: 2} show that in the sub-population ``violation" extracted by \textbf{mdfa}, the distribution of features is narrower and similar for African-American and  non-African American: the sub-population is made of individuals with little criminal and misdemeanor history. However, African American are still three times more likely to be classified as high risk. A policy implication of \textbf{mdfa} findings is that a judge using COMPAS may discount its assessment for African-American with little criminal history.

\begin{table} [t]
	\centering  
	\begin{footnotesize}
	\begin{tabular}{c|cc||cc} 
		Variable & \multicolumn{2}{c}{Population} & \multicolumn{2}{c}{ Violation} \\
		& AA & Other & AA& Other \\
		\hline 
		\hline 
		Prior Felonies & 4.44 & 2.46 & 0.79 & 0.67 \\
		 &  $\scriptsize{(5.58)}$ & $\scriptsize{(3.76)}$ & $\scriptsize{(0.24)}$ & $\scriptsize{(0.17)}$ \\
		Charge Degree & 0.31  & 0.4  & 0.74  & 0.74  \\ 
		 & $\scriptsize{(0.46)}$ & $\scriptsize{(0.49)}$ &  $\scriptsize{(0.23)}$ & $\scriptsize{(0.2)}$  \\
		Juvenile Felonies & 0.1 & 0.03 & 0.01 & 0.0  \\
		 & $\scriptsize{(0.49)}$ & $\scriptsize{(0.32)}$ & $\scriptsize{(0.02)}$ & $\scriptsize{(0.02)}$  \\
		Juvenile Misdemeanor & 0.14  & 0.04 & 0.01 & 0.01  \\ 
		 & $\scriptsize{(0.61)}$ & $\scriptsize{(0.3)}$ & $\scriptsize{(0.02)}$ & $\scriptsize{(0.01)}$ \\ 
		\hline
		High Risk & 0.14 & 0.05 & 0.06 & 0.02 \\
		 & $\scriptsize{(0.35)}$ & $\scriptsize{(0.22)}$ & $\scriptsize{(0.04)}$ & $\scriptsize{(0.01)}$ \\
	\end{tabular} 
	\caption{Identifying the worst-case violation of differential fairness in the COMPAS risk score. The sensitive attribute is whether the individual is self-identified as African American ($AA$) or not (Other). ( ) indicates standard deviation.}
	\label{tab: 2}
	\end{footnotesize}
\end{table}

\subsection{Group Fairness vs. Multi-Differential Fairness}
We evaluate whether previous fairness correcting approaches  protect small group of individuals against violation of differential fairness. We consider two techniques: (i) \cite{feldman2015certifying}'s disparate impact with a logistic classification ($DI-LC$) and (ii)   \cite{agarwal2018reductions}'s reduction with a logistic classification ($Red-LC$). We use \textbf{mdfa} to identify sub-population $G$  with worst-case violations  and measure disparate treatment as $DT_{G}=Pr[Y=1| S=1, G]/Pr[Y=1| S=-1, G]$. We compare $DT_{G}$ to its aggregate counterpart computed on the whole population $DI=Pr[Y=1| S=1]/Pr[Y=1| S=-1]$. 

\paragraph{Data}
The experiment is carried on three datasets from \cite{friedler2018comparative,kearns2018empirical}): \textbf{Adult} with $48840$ individuals; \textbf{German} with $1000$ individuals; and, \textbf{Crimes} with $1994$ communities. In \textbf{Adult} the prediction task is whether an individual's income is less than $50K$ and the sensitive attribute is gender; in \textbf{German}, the prediction task is whether an individual has  bad credit and the sensitive attribute is gender; in \textbf{Crimes}, the task is to predict whether a community is in the $70^{th}$ percentile for violent crime rates and the sensitive attribute is whether the percentage of African American is at least $20\%$. For each data, each repair technique produces a prediction; then, \textbf{mdfa} is trained on $70\%$ of the data and computes estimates for disparate treatment $DT_{G}$ on the remaining $30\%$ of the data. The experiment is repeated with $100$ train/test splits. 

\paragraph{Results}
In Table \ref{tab: 3}, even despite the fairness correction applied by $DI-LC$ and $Red-LC$, \textbf{mdfa} still finds sub-populations $G$ for which $DT_{G}$ is significantly larger than one. It indicates the existence of group of individuals who are similar but for their sensitive attributes and who are treated differently by the classifier trained by either $DI-LC$  or  $Red-LC$. The repair techniques reduce the aggregate disparate impact compared to the baseline ($LC$), since $DI$ is closer to one for $DI-LC$ and $Red-LC$ across all datasets. However, in the \textbf{Adult} dataset, $DT_{G}$ remains between $1.44$ and $1.6$ after repair: \textbf{mdfa} identifies a group $G$ of Females that are $44\%-60\%$ more likely to be of low-income than Males with similar characteristics. In \textbf{Crimes} dataset, disparate treatment $DT_{G}$ is around $5.7$ for both $DI-LC$, $R-LC$: this means that there exist communities with dense African-American populations that are six times more likely to be classified at high risk than similar communities with lower percentages of African Americans.

In Table \ref{tab: 4}, we identifies the average characteristics of the worst-case violation sub-population $G$ in the \textbf{Adults} dataset. For brevity, we report the results for the logistic classifier without repair $LC$ and with \cite{agarwal2018reductions}'s reduction repair $Red-LC$. Similar results can be obtained for $DI-LC$. Compared to the overall populations, for both $LC$ and $Red-LC$, individuals in the worst-case violation sub-population work more hours/week, are older and have more years of education. Women in that group are $60\%-80\%$ more likely to be classified as low-income by $LC$ or $Red-LC$ than men with the same high level of education, same hours of work per week and same age.  

\begin{table} [t]
	\centering 
	\begin{footnotesize}
	\begin{tabular}{p{1.25cm}|cc|cc|cc} 
		Repair & \multicolumn{2}{c}{Adult} & \multicolumn{2}{c}{German} & \multicolumn{2}{c}{Crimes} \\
		Technique & $DT_{G}$ & $DI$ & $DT_{G}$ & $DT$ & $DT_{G}$ & $DI$ \\ 
		\hline\hline
		LC & $1.88$ & 1.08 & $1.26$ & 1.07 & $5.76$ & 1.0 \\ 
		 & $\scriptsize{(0.4)}$ &  & $\scriptsize{(0.14)}$ &  & $\scriptsize{(3.16)}$ &\\ 
		DI-LC & $1.44 $ & 0.99 & $1.1 $ & 1.04 & $5.74 $ & 1.0 \\
		 & $\scriptsize{(0.32)}$ &  & $\scriptsize{(0.08)}$ &  & $\scriptsize{(2.19)}$ &  \\
		Red-LC & $1.6 $ & 1.03 & $1.04$ & 1.01 & $5.24 $ & 1.0 \\ 
		 & $\scriptsize{(0.25)}$ &  & $ \scriptsize{(0.21)}$ &  & $ \scriptsize{(0.89)}$ &  \\ 
	\end{tabular} 
	\caption{Worst-case violations of multi-differential fairness identified by \textbf{mdfa} for classifiers trained with standard fairness repair techniques. ( ) indicates standard deviation.}
	\label{tab: 3} 
	\end{footnotesize}
\end{table}  


\begin{table}[!t] 
\centering  
\begin{footnotesize}
\begin{tabular}{p{1.75cm}|p{0.6cm}p{0.6cm}||p{0.6cm}p{0.6cm}||p{0.6cm}p{0.6cm}} 
Variable & \multicolumn{2}{c}{Population} & \multicolumn{4}{c}{Worst-case Violation} \\ 
 & \multicolumn{2}{c}{} & \multicolumn{2}{c}{LC} & \multicolumn{2}{c}{Red-LC} \\ 
 & F& M & F& M & F& M \\ 
\hline 
\hline 
Level of& 10.41 & 10.2 & 10.74 & 10.84 &10.31&10.51\\ 
 Education& (\scriptsize{3.94}) & (\scriptsize{3.85}) & (\scriptsize{0.11}) & (\scriptsize{0.06}) & (\scriptsize{0.16}) & (\scriptsize{0.10}) \\ 
 \hline
Years of & 10.06 & 10.08 & 13.74 & 13.9 & 13.08 & 13.32\\ 
Education & (\scriptsize{2.38}) & (\scriptsize{2.66}) & (\scriptsize{0.04}) & (\scriptsize{0.04}) & (\scriptsize{0.13}) & (\scriptsize{0.08}) \\
\hline
Hours/week & 36.38 & 42.39 & 44.77 & 48.02 & 49.32 & 51.44 \\ 
 & (\scriptsize{12.22}) & (\scriptsize{12.12}) & (\scriptsize{1.19}) & (\scriptsize{0.75}) & (\scriptsize{0.94}) & (\scriptsize{0.51})\\ 
 \hline
Occupation & 6.2 & 6.78 & 7.75 & 7.93 & 6.88 & 7.65 \\ 
 & (\scriptsize{4.39}) & (\scriptsize{4.14}) & (\scriptsize{0.07}) & (\scriptsize{0.11}) & (\scriptsize{0.30}) & (\scriptsize{0.16}) \\ 
 \hline
Age & 37.07 & 39.62 & 49.72 & 50.47 & 49.48 & 48.46 \\ 
 & (\scriptsize{14.38}) & (\scriptsize{13.50}) & (\scriptsize{1.28}) & (\scriptsize{0.60})  & (\scriptsize{1.31}) & (\scriptsize{0.88}) \\ 
 \hline
\end{tabular} 
\caption{\textbf{Adults}: Identifying the worst-case violation of differential fairness for classifiers trained with $LC$ and $Red-LC$. The sensitive attribute is whether the individual is self-identified as Female ($F$) or Male ($M$).( ) indicates standard deviation.}
\label{tab: 4} 
\end{footnotesize}
\end{table}

\section{Conclusion}
In this paper, we present \textbf{mdfa}, a tool that measures whether a classifier treats differently individuals with similar auditing features but different sensitive attributes. We hope that \textbf{mdfa}'s ability to identify sub-populations with severe violations of differential fairness could inform decision-makers when to discount the classifier's outcomes. It also provides the victims with a framework to contest a classifier's outcomes. 

Avenues for future research are to investigate (i)  the properties of a classifier trained under a multi-differential fairness constraint; and, (ii) the possibility to extend our approach to re-balance distributions in order to make counterfactual inference \cite{johansson2016learning} in the context of algorithmic fairness. 

\section{Appendix}
\paragraph{Lemma \ref{lem: 1}}
\begin{proof}
Denote $\langle x,x^{'}\rangle$ the inner product between $x$ and $x^{'}$. Observe that for $r=\pm$ the left-hand side in Eq. \eqref{eq: unfair} can be written $\frac{1}{2}\left\langle\frac{c+1}{2},  S\frac{1+rY}{2}\right\rangle$ since $\langle x,x^{'}\rangle = Pr_{w}[x=x^{'}] -1$ for any $x, x^{'}\in \{-1, 1\}$. The result from lemma \ref{lem: 1} follows by remarking that $\langle S, 1\rangle = \langle S, c\rangle = 2 Pr_{w}[S=c] -1 = 0$, since $Pr_{w}[S=s|x] = Pr_{w}[S\neq s|x]$. 
\end{proof}

\paragraph{Theorem \ref{thm: al}}
\begin{proof}
$(i)\Rightarrow (ii)$. Denote $(x_{i}, s_{i}, o_{i})$ a sample from a balanced distribution $D$ over $\mathcal{X}\times \mathcal{S} \times \{-1, 1\}$.  Denote $c^{*}\in \mathbb{C}$ such that $Pr[c^{*}(x_{i})=o_{i}] = max_{c\in \mathbb{C}}Pr[c(x_{i})=o_{i}]=opt$. Construct a function $f$ such that for $(x_{i}, s_{i}, o_{i})$, $f(x_{i}, s_{i}) = s_{i}o_{i}$. Therefore, $f(x_{i})s_{i}=o_{i}$ and  $Pr[c^{*}=s_{i}f(x_{i})]= Pr[c^{*}=o_{i}]=opt$: by lemma \ref{lem: 1}, $c^{*}$ is a $\gamma$-unfairness certificate, with $\gamma=\frac{opt + \rho -1}{4}$ and $\rho=Pr[o_{i}=1]$. By $(i)$, the certifying algorithm outputs a $(\gamma -\epsilon/4)-$ unfairness certificate $c\in \mathbb{C}$  with probability $1-\eta$ and $O(\log(|\mathcal{C}, \log(\frac{1}{\eta}), \frac{1}{\epsilon^{2}})$ sample draws. Hence, by lemma \ref{lem: 1}, $Pr[c(x_{i})=o_{i}]=Pr[c(x_{i})=f(x_{i})o_{i}] = 4(\gamma - \epsilon/4) + 1 -\rho =opt - \epsilon$, which concludes $(i)\Rightarrow (ii)$

$(ii)\Rightarrow (i)$. Suppose that $f$ is a $\gamma$-unfair. Denote $y_{i}=f(x_{i}, s_{i})$. Samples $\{(x_{i}, s_{i}), y_{i}\}$ are drawn from a balanced distribution over $\mathcal{X}\times \mathcal{S}\times \{-1, 1\}$. By lemma \ref{lem: 1}, there exists $c\in \mathbb{C}$ such that $Pr[c(x_{i})=s_{i}y_{i}]=4\gamma + 1 - \rho_{r}$, with $r=\pm$. Assume, without loss of generality $r=+$. Then, since $\max_{c^{'}}Pr[c(x_{i})=s_{i}y_{i}] \geq 4\gamma + 1 - \rho_{+}$. By $(ii)$, there exists an algorithm that outputs with probability $1-\eta$ and   $O(\log(|\mathcal{C}, \log(\frac{1}{\eta}), \frac{1}{\epsilon^{2}})$ sample draws $c\in \mathbb{C}$ such that $Pr[c(x_{i})=s_{i}y_{i}] \geq \max_{c^{'}}Pr[c(x_{i})=s_{i}y_{i}] - \epsilon/4$. Therefore $Pr[c(x_{i})=s_{i}y_{i}] \geq 4(\gamma -\epsilon) + 1 - \rho_{+}$. By lemma \ref{lem: 1}, $c$ is a $(\gamma-\epsilon)-$ unfairness certificate for $f$, which concludes $(ii)\Rightarrow (i)$.
\end{proof}

\paragraph{Theorem \ref{thm: corr1}}

We first show the following lemma
\begin{lem}
\label{lem: 2}
With the same assumption as in lemma \ref{lem: 3}, for any weights $u, w$, 
$||u-w|| \leq G_{k}(u, w)/\sqrt{\lambda_{min}(k)},$ where $\lambda_{min}(k)$ is the smallest eigenvalue of the Gram matrix associated with $k$
\end{lem}

\begin{proof}
First, note that $G_{k}(u, w)=\sqrt{(u-w)^{T}k(u-w)}$  and by a standard bound on Rayleigh quotient, $||u-w|| \leq G_{k}(u, w)/\sqrt{\lambda_{min}(k)}$.
\end{proof}

 Now we prove lemma \ref{lem: 3}:
 \begin{proof}
 The proof relies on the fact that the solution of Eq. \eqref{eq: risk1} is distributionally stable as in \cite{cortes2008sample}:
 \begin{equation}
 \nonumber
     ||h_{u}-h_{w}||_{k} \leq \kappa\sigma^{2}\frac{\sqrt{\lambda_{max}(k)}}{2\lambda_{c}}||u-w||, 
 \end{equation}
 where $\lambda_{max}(k)$ is the largest eigenvalue of the Gram matrix associated to $k$ (see \cite{cortes2008sample}, proof of theorem 1). Moreover, $|h_{u}(x) - h_{w}(x)| \leq ||h_{u}-h_{w}||_{k}$. The result in lemma \ref{lem: 3} follows from lemma \ref{lem: 2}  and $cond(k)=\lambda_{max}(k) /\lambda_{min}(k)$. 
 \end{proof}

 The next result follows from \cite{gretton2009covariate} and bounds above $\hat{G}_{k}(u, w_{s})$, the emprical counterpart of $G_{k}(u, w_{s})$, where $w_{s}(x)=Pr[S\neq s|x]/(1-Pr[S=s|x])$ are the importance sampling weights for $S=s$. 
 
 \begin{lem}
 \label{lem: 4}
 Let $\eta > 0$. Denote $n_{s}=|\{i=1,...,m|s_{i}=s\}|$ and $n_{s}=|\{i=1,...,m|s_{i}\neq s\}|$. Suppose that $||u||_{\infty} < B/n_{s}$ and that $E[u] <\infty$. There exists a constant $\kappa_{1}>0$ such that with probability $1-\eta$,
 $$\hat{G}_{k}(u, w_{s}) \leq \kappa_{1}\sqrt{2\log\frac{2}{\eta}\left(\frac{B^{2}}{n_{s}} + \frac{1}{n_{\neg s}}\right)}$$.
 \end{lem}
 
 \begin{proof}
 See \cite{gretton2009covariate} Lemma 1.5
 \end{proof}
 
 For $\tau>0$, we construct $c\in \mathbb{C}$  from the real-valued function $h$ as 
 \begin{equation}
     c(x) =\begin{cases} sign(h) & \mbox{if } |h| >\tau \\
     1  \mbox{ w.p.} \frac{h+\tau}{2\tau} &\mbox{if } |h|\leq \tau
     \end{cases}
 \end{equation}
 
 \begin{lem}
 \label{lem: 5}
 Let $\tau, \eta >0, \epsilon >0$. Let $\hat{c}_{u}$, $\hat{c}_{w}$ denote the certificates constructed from the solution of the empirical risk minimization $\hat{h}_{u}$ (with weights $u$) and $\hat{h}_{w}$ (with weights $w_{s}$). There exists $\kappa_{2}>0$ such that with probability $1-\eta$, $$\frac{|i=1,..., m| \hat{c}_{u}(x_{i})\neq \hat{c}_{w}(x_{i})|}{m}\leq \kappa_{2}\sqrt{2\log\frac{2}{\eta}\left(\frac{B^{2}}{n_{s}} + \frac{1}{n_{\neg s}}\right)}$$.
 \end{lem}
 
 \begin{proof}
 Denote  $\epsilon(m)=\frac{\kappa_{2}\tau}{5}\sqrt{2\log\frac{2}{\eta}\left(\frac{B^{2}}{n_{s}} + \frac{1}{n_{\neg s}}\right)}$, with $\kappa_{2}=\kappa_{1}\kappa\sigma^{2}\frac{\sqrt{cond(k)}}{2\lambda_{c}}$. By lemma \ref{lema: 3} and lemma \ref{lem: 4}, $||\hat{h}_{u}-\hat{h}_{w}||_{2}\leq \epsilon_{m}$ with probability $1-\eta$. Consider first the case $\hat{h}_{w}<-\tau$. Then, with probability $1-\eta$, $\{x_{i}|\hat{c}_{u}(x_{i})\neq \hat{c}_{w}(x_{i}) \land \hat{h}_{w}(x_{i})<-\tau\} = \{x_{i}|(-\tau + \epsilon_{m} > \hat{h}_{u}(x_{i})> -\tau) \land (\hat{c}_{u}(x_{i}=1) \land \hat{h}_{w}(x_{i})<-\tau]$ by lemma \ref{lem: 3}. Therefore, by construction of $\hat{c}_{u}$, $|\{x_{i}|\hat{c}_{u}(x_{i})\neq \hat{c}_{w}(x_{i}) \land \hat{h}_{w}(x_{i})<-\tau\}|  = m \frac{\tau +h_{u}}{2\tau}\leq m\frac{\epsilon_{m}}{5\tau}$. A similar result is obtained for $\hat{h}_{w} > \tau$. .

 Lastly, with probability $1-\eta$, $|\{x_{i}|\hat{c}_{u}(x_{i})\neq \hat{c}_{w}(x_{i}) \land \hat{h}_{w}(x_{i})\in(-\tau, \tau)\}|\leq 2 m \frac{\epsilon_{m}}{5\tau} + m \frac{|h_{u}-h_{w}|}{\tau} \leq 3\frac{\epsilon_{m}}{5\tau}$. The first part of the inequality uses the results obtained in the previous paragraph for $|\hat{h}_{u}|>\tau$;  the second part uses the construction of $\hat{c}_{u}$ and $\hat{c}_{w}$. 
 
 Therefore, with probability $1-\eta$,  $\frac{|i=1,...,m| \hat{c}_{u}(x_{i})\neq \hat{c}_{w}(x_{i})|}{m}\leq 5\epsilon(m)/\tau.$   
 \end{proof}
 
 \begin{lem}
 \label{lem: 6}
 Consider a random variable $Z$ and $\hat{c}_{u}$, $\hat{c}_{w}$ as in lemma \ref{lem: 5}. Then $|Pr[Z=1 \land \hat{c}_{u}=1] - Pr[Z=1 \land \hat{c}_{w}=1]| \leq Pr[\hat{c}_{u}\neq \hat{c}_{w}]$
 \end{lem}
 
 \begin{proof}
 Note that $Pr[Z=1 \land \hat{c}_{u}=1] - Pr[Z=1 \land \hat{c}_{w}=1] = Pr[Z=1 \land \hat{c}_{u}=1 \land \hat{c}_{w}=-1] - Pr[Z=1 \land \hat{c}_{w}=1 \land \hat{c}_{u}=-1]\leq Pr[\hat{c}_{u}\neq \hat{c}_{w}]$.
 \end{proof}
 
 The last result we need to prove theorem \ref{thm: corr1} is to link $n_{s}$ and $n_{\neg s}$ to sample size $m$
 
 \begin{lem}
 \label{lem: 7}
 Denote $\alpha_{s}=Pr[S=s]$. Let $\epsilon, \eta >0$. Therefore,  if $m\geq \Omega\left(\frac{\log(1/\eta)}{\alpha_{s}^{2}\epsilon^{2}}\right)$, with probability $1-\eta$, $n_{s}\geq \alpha_{s}(1 -\epsilon)m$. 
 \end{lem}
 
 \begin{proof}
 This is an application of a Hoeffding's inequality for Bernouilly random variable.
 \end{proof}
 
 The proof of theorem \ref{thm: corr1} then follows from the observation that for a sample balanced with weights $u$, $\hat{\gamma}_{u} = Pr[\hat{c}_{u}=1]\left(\frac{e^{\delta_{u}}}{e^{\delta_{u}}+1}-\frac{1}{2}\right)$ with $e^{\delta_{u}}=Pr[Y=1|S=1, \hat{c_{u}}=1]/Pr[Y=1|S=-1, \hat{c}_{u}=1]$. By lemmas \ref{lem: 6} and \ref{lem: 7}, for any $\epsilon^{'}>0$ with $\Omega\left(\frac{1}{(\epsilon^{'})^{2} \alpha_{s}^{2}}\log\frac{2}{\eta}\right)$ samples, with probability $1-\eta$, $$\left|\frac{Pr[Y=1|S=1,\hat{c}_{w}=1]}{Pr[Y=1|S=-1, \hat{c}_{w}=1]} - \frac{Pr[Y=1|S=1, \hat{c}_{u}=1]}{Pr[Y=1|S=-1, \hat{c}_{u}=1]} \right| \leq \epsilon^{'}$$ and 
 
 $$|Pr[\hat{c}{w}=1 \land Y=1] - Pr[\hat{c}_{u}=1 \land Y=1]| \leq \epsilon^{'}.$$
There exists $\kappa_{3}$ such that for $\epsilon >0$, with probability $1-\eta$, $|\hat{\gamma}_{u}- \hat{\gamma}_{w}| \leq \kappa_{3}\epsilon$. Moreover, by theorem \ref{thm: al}, if $\mathbb{C}$ has finite VC dimension, with probability $1-\eta$ and  $\Omega\left(\log(|\mathbb{C}|),\frac{1}{\epsilon^{2} },\log\frac{2}{\eta}\right)$ samples$, |\hat{\gamma}_{w}-\gamma_{w}| \leq \epsilon$. It follows that with probability $1-\eta$ and $\Omega\left(\log(|\mathbb{C}|),\frac{1}{\epsilon^{2}\alpha_{s}^{2} },\log\frac{2}{\eta}\right)$ samples, $|\hat{\gamma}_{u}- \gamma_{w}|\leq (1+ \kappa_{3})\epsilon$. That concludes the proof since by lemma \ref{lem: 1}, if $f$ is $\gamma-$unfair, then $\gamma=\gamma_{w}$.

 \paragraph{Theorem \ref{thm: algo3_ana}}
 Assume that the classifier $f$ is $\gamma$-unfair for $Y=y\in \{-1, 1\}$. Let $\delta_{m}$ denote the worst-case violation. Note that by definition of $\gamma$ and $\delta_{m}$:
$
\gamma = \alpha\left(\frac{e^{\delta_{m}}}{e^{\delta_{m}}+ 1}- \frac{1}{2}\right)
$. At each iteration $t$, denote $c_{t}$ the solution of the following optimization problem 
\begin{equation}
\label{eq: opt1}
max_{c\in \mathbb{C}} E\left[\displaystyle\sum_{i=1}^{m} u_{it}\mathbbm{1}_{a}\left(s_{i}y_{i}=c(x_{i})\right)\right],
\end{equation}

where $u_{it}$ are the weights at iteration $t$ and the expectation is taken over all the samples of size $m$ drawn from $D_{f}$.
\begin{equation}
u_{it} = \begin{cases}
u_{i}(1 + \nu t) \mbox{ if } y_{i}\neq s_{i} \land y_{i}=y \\
u_{i} \mbox{ otherwise.}
\end{cases}
\end{equation}

 \begin{lem}
 Let $s\in\mathcal{S}$. Assume that the classifier $f$ is $\gamma$-unfair. At iteration $t$, denote $\delta_{t}=\ln(Pr[Y=1|c_{t}(x)=1,S=s]/Pr[Y=1|c_{t}(x)=1,S=s])$. Then,
 \begin{equation}
 \nonumber
\frac{e^{\delta_{t}}}{1 + e^{\delta_{t}}} \geq 1 -h(\xi t),
\end{equation}
where $h(\xi t) = \frac{4\gamma + 1 - 2\rho_{+}}{\xi t}$.
 \end{lem}

\begin{proof}
Without loss of generality, we assume$y=1$. Denote $c_{-1}$ such that $c_{-1}(x)=-1$ for all $x$. By comparing the value of the empirical risks for $c_{t}$ and $c_{-1}$, if $c_{t}\neq c_{-1}$, then we can show that 
\begin{equation}
\nonumber
\begin{split}
    Pr_{u}[c_{t}(x_{i})= s_{i}y_{i})] & \geq Pr_{u}[s_{i}\neq y_{i}]   \\
    &+\xi t  E\left[\displaystyle\sum_{\substack{i=1, s_{i}\neq y_{i} \\ c_{t}(x_{i})=1 \\ y_{i}=y} }^{m} u_{i}\right]. 
\end{split}
\end{equation}
Moreover, 
\begin{equation}
\nonumber
E\left[\displaystyle\sum_{\substack{i=1, s_{i}\neq y_{i} \\ c_{t}(x_{i})=1) \\ y_{i}=y}}^{m} u_{i}\right] = Pr_{u}[s_{i}y_{i}\neq c_{t}(x_{i})\land  c_{t}(x_{i})=1 \land y_{i}=y].
\end{equation}
It follows that if $c_{t}\neq c_{-1}$, since $Pr_{u}[c_{t} =1, Y=y] \geq \alpha$
$Pr_{u}[SY=c_{t}|c_{t}=1, Y=y] \geq 1 - \frac{1}{\alpha \xi t}\left(Pr_{u}[c_{t}=SY)] - \rho_{+}\right)$. Since the classifier $f$ is $\gamma$-unfair for $y=1$, we know that 
$\max_{c\in\mathbb{C}}Pr_{u}[SY=c] = 4\gamma + 1 - \rho_{+}.$ Therefore, at iteration $t$, either $c_{t}=c_{-1}$ or 
\begin{equation}
\nonumber
\label{eq: bound}
Pr_{u}[SY=c_{t}|c_{t}=1, Y=y] \geq 1 - \frac{4\gamma + 1 - 2\rho_{+}}{\xi\alpha t}= 1-h(\xi t). 
\end{equation}
Since $y=1$, $Pr_{u}[Y=1|c_{t}=1, S=1] \geq 1-h(\xi t)$. Without loss of generality, we can assume $Pr[Y=1|c_{t}=1, S=1] \geq 1-h(\xi t)$. It follows that $e^{\delta_{t}} \geq \frac{1-h(\xi t)}{h(\xi t)}$. 
\end{proof} 

\begin{lem}
Denote $T = \frac{4\gamma + 1-2\rho(y)}{\xi\alpha} \left(e^{\delta_{m}} + 1\right)$. The algorithm stops for $t \geq T$ with $Pr[Y=y, C_{t}=1]=\alpha$ and $\delta_{t}=\delta_{m}$.
\end{lem}

\begin{proof}
First, note that $h(\xi T)=\frac{1}{e^{\delta_{m}} + 1}$ and thus that $\delta_{T}\geq \delta_{m}$.  Moreover, at iteration $t$, $c_{t}$ is chosen over $c_{m}$ where $c_{m}$ is the sub-population that corresponds to the worst-case violation $\delta_{m}$. Comparing the expected risk for $c_{m}$ and $c_{t}$ at iteration $t$ leads to 

\begin{equation}
\label{eq: opt_bound}
\begin{split}
2Pr_{u}[SY=c_{t}] - 2Pr_{u}[SY=c_{m}] & \\ \geq \xi t E\left[\displaystyle\sum_{\substack{i=1\\ s_{i}\neq y_{i} \\ c_{t}(x_{i})=1 \\y_{i}=1}}^{m}  u_{i} - \displaystyle\sum_{\substack{i=1\\ s_{i}\neq y_{i} \\ c_{m}(x_{i})=1 \\y_{i}=1}}^{m}  u_{i}\right] & \\
 = \xi t \left(Pr_{u}[c_{t}=1 \land S\neq Y \land Y=1] -  \right. &\\ 
 \left. Pr_{u}[c_{m}=1\land S\neq Y \land Y=1]\right) & \\
 = \left(Pr_{u}[SY\neq c_{t} |c_{t}=1, Y=1]Pr_{u}[c_{t}=1, Y=1] \right. & \\
- \left. Pr[SY\neq c_{m}|c_{m}=1, Y=1]Pr[c_{m}=1, Y=1]  \right) &\\
\end{split}
\end{equation}
By definition of the worst-case violation for $y=1$, $Pr_{u}[SY\neq c_{t}|c_{t}=1, Y=1] \geq Pr_{u}[SY\neq c_{m}|c_{m}=1, Y=1]$. Moreover, when the algorithm stops, $Pr_{u}[c_{t}=1, Y=1]=\alpha=Pr_{u}[c_{m}=1, Y=1]$. Therefore, the right-hand side of \eqref{eq: opt_bound} is non-negative. It results that $Pr_{u}[SY=c_{t}] \geq Pr_{u}[SY=c_{m}]$. On the other hand, since $f$ is $\gamma-$ unfair, $Pr_{u}[SY=c_{t}]$ cannot be more than $4\gamma - \rho_{+} + 1$. Therefore, $Pr_{u}[SY=c_{t}] = 4\gamma - \rho_{+} + 1$. 

It follows that at iteration $t$, when the algorithm stops
\begin{equation}
\nonumber
Pr_{u}[Y=1,c_{t}=1]\left(\frac{e^{\delta(c_{t})}}{1 + e^{\delta(c_{t})}}-\frac{1}{2}\right)=\gamma.
 \end{equation}
The algorithm stops when $Pr[Y=1,c_{t}=1] =\alpha$, which implies $\delta_{t}=\delta_{m}$. Therefore, $t\leq T$. 
\end{proof}

\clearpage
\bibliographystyle{named}
\bibliography{ijcai19}

\end{document}